\RequirePackage{amsmath}
\documentclass[12pt]{llncs}

\usepackage[utf8]{inputenc}
\usepackage[T1]{fontenc}
\usepackage{latexsym}
\usepackage{amssymb}
\usepackage{lmodern}
\usepackage{fix-cm}
\usepackage{url}
\usepackage{array}

\usepackage{graphicx}
\usepackage{tabularx}
\usepackage{array}
\usepackage{mdwmath}
\usepackage{mdwtab}
\usepackage{float}
\usepackage{tikz}
\usetikzlibrary{shapes,arrows}
\usepackage{makecell}
\usepackage{diagbox}

\usepackage{stmaryrd}
\usepackage{eulervm}
\usepackage{palatino}
\usepackage{booktabs}

\usepackage{algorithm2e}

\newcommand{\kap}{\kappa}

\newcommand{\pc}{\mathbf{P}}

\newcommand{\pp}{\mathbb{P}}

\begin{document}
\setcounter{page}{1}
 
\title{Algebraic Semantics of Generalized RIFs}
\author{\textsf{A Mani}}
\institute{Indian Statistical Institute\\
203, B. T. Road, Kolkata-700108, India\\
Email: \texttt{$a.mani.cms@gmail.com$, $amani.rough@isical.ac.in$}\\
Homepage: \url{https://www.logicamani.in}\\
Orcid: \url{https://orcid.org/0000-0002-0880-1035} }

\maketitle

\begin{abstract}
A number of numeric measures like rough inclusion functions (RIFs) are used in general rough sets and soft computing. But these are often intrusive by definition, and amount to making unjustified assumptions about the data. The contamination problem is also about recognizing the domains of discourses involved in this, specifying errors and reducing data intrusion relative to them. In this research, weak quasi rough inclusion functions (wqRIFs) are generalized to general granular operator spaces with scope for limiting contamination. New algebraic operations are defined over collections of such functions, and are studied by the present author. It is shown by her that the algebras formed by the generalized wqRIFs are ordered hemirings with additional operators. By contrast the generalized rough inclusion functions lack similar structure. This potentially contributes to improving the selection (possibly automatic) of such functions, training methods, and reducing contamination (and data intrusion) in applications. The underlying framework and associated concepts are explained in some detail as they are relatively new. 
\end{abstract}


\section{Introduction}

If $A, B\in\mathcal{S}\subseteq \wp(S)$, with $\mathcal{S}$ being closed under intersection, $S$ being a finite set and if $\# ()$ is the cardinality function, then the quantity 
\begin{equation*}\label{rif0}
\nu(A, B) = \left\lbrace  \begin{array}{ll}
 \dfrac{\# (A\cap B)}{\# (A)} & \text{if } A\neq \emptyset\\
 1 & \text{if } A= \emptyset\\
 \end{array} \right. \tag{K0}                                                                                                             
\end{equation*}
can be interpreted in multiple ways (and therefore associated with distinct ontologies) including as rough inclusion function, conditional subjective probability, relative degree of misclassification, majority inclusion function, and inclusion degree. In this it is possible to replace \emph{intersections} with commonality operations that need not be idempotent, commutative or even associative. Many generalizations of this function are known in the rough set, belief theory, subjective probability, fuzzy set and ML literature. It leads to ideas of concepts being close or similar to each other in the contexts of rough sets -- but subject to a number of hidden conditions. At the same time this is at conflict with concepts of association reducts in rough sets \cite{sdrskt06,amst96}. A major problem with such functions is that they are not dependent on granules (in multiple perspectives) \cite{am501,am240}. But solving this is bound to be a very complicated matter in any perspective. Abstract granular frameworks of general rough sets are used to address the issue in this research.

Data analysis maybe intrusive (invasive) or non-intrusive relative to the assumptions made on the dataset used in question \cite{gdu}. Non-invasive data analysis was defined in a vague way in \cite{gdu} as one that 
\begin{itemize}
\item {is based on the idea of \emph{data speaking for themselves},}
\item {uses minimal model assumptions by drawing all parameters from the observed data, and}
\item {admits ignorance when no conclusion can be drawn from the data at hand.}
\end{itemize}
Key procedures deemed to be non-invasive are data discretization (or horizontal compression), randomization procedures, reducts of various kinds within rough set data analysis, and rule discovery with the aid of maximum entropy principles. In most situations, general rough inclusion functions are also source of data intrusion and contamination (which is about using assumptions of one domain in another during modeling). New methods of reducing invasive procedures and contamination are a consequence of this research.

Very general variants of granular operator spaces (specifically GGS and specific versions thereof) \cite{am5559,am501,am240,am9114} are used as the basic framework. Ideas of contamination are explained in brief and an application to possible soft walking aids is also discussed. Weak quasi rough inclusion functions are generalized and studied from an algebraic perspective. Potential application to model selection, decision-making and cluster validity (in the context of \cite{am2021c}) are also indicated by the present author.

The paper is arranged as follows. In the following section, some of the necessary background is defined for convenience. The concept of contamination and data intrusion is discussed in the third section. Variants of granular operator spaces are explained  in the next. Generalized rough inclusions of various types are defined and basic properties are proved in the fifth. Admissible algebraic operations and systems over generalized wqRIFs and RIFs are investigated in the sixth section.    
 
\section{Background and Notation}\label{bck}

For the basics of rough sets, the reader is referred to \cite{ppm2,lp2011}. The granular approach due to the present author can be found in \cite{am5559,am501,am240,am9114}. From a set-theoretic perspective, \emph{granulations} are subsets of a powerset that satisfies some conditions. In \cite{tyl}, a model of granulation (GrC model) is defined as a collection of objects of a category along with a set of n-ary relations (subobject of a product object). The present usage is obviously distinct (though related).

\emph{Throughout the paper, quantification is enclosed in braces for easier reading;
\[(\forall a_1,\ldots , a_n) \text{ is the same as } \forall a_1,\ldots ,a_n\]}.

Information tables are structured forms of data. Formally, an \emph{information table} $\mathcal{I}$, is a relational system of the form \[\mathcal{I}\,=\, \left\langle \mathfrak{O},\, \mathbb{A},\, \{V_{a} :\, a\in \mathbb{A}\},\, \{f_{a} :\, a\in \mathbb{A}\}  \right\rangle \]
with $\mathfrak{O}$, $\mathbb{A}$ and $V_{a}$ being respectively sets of \emph{Objects}, \emph{Attributes} and \emph{Values} respectively.
$f_a \,:\, \mathfrak{O} \longmapsto \wp (V_{a})$ being the valuation map associated with attribute $a\in \mathbb{A}$. Values may also be denoted by the binary function $\nu : \mathbb{A} \times \mathfrak{O} \longmapsto \wp{(V)} $ defined by for any $a\in \mathbb{A}$ and $x\in \mathfrak{O}$, $\nu(a, x) = f_a (x)$.

Relations may be derived from information tables by way of conditions of the following form: For $x,\, w\,\in\, \mathfrak{O} $ and $B\,\subseteq\, \mathbb{A} $, $(x,\,w)\,\in\, \sigma $ if and only if $(\mathbf{Q} a, b\in B)\, \Phi(\nu(a,\,x),\, \nu (b,\, w),) $ for some quantifier $\mathbf{Q}$ and formula $\Phi$. The relational system $S = \left\langle \underline{S}, \sigma \right\rangle$ (with $\underline{S} = \mathbb{A}$) is said to be a \emph{general approximation space}. 

If $\sigma$ is an equivalence as defined by the condition Equation \ref{pawl} then $S$ is said to be an \emph{approximation space}.
\begin{equation*}\label{pawl}
(x, w)\in \sigma \text{ if and only if } (\forall a\in B)\, \nu(a,\,x)\,=\, \nu (a,\, w) 
\end{equation*}

In classical rough sets, on the power set $\wp (S)$, lower and upper
approximations of a subset $A\in \wp (S)$ operators, apart from the usual Boolean operations, are defined as per: 
\[A^l = \bigcup_{[x]\subseteq A} [x] \; ; \; A^{u} = \bigcup_{[x]\cap A\neq \varnothing } [x],\,\]
with $[x]$ being the equivalence class generated by $x\in S$. If $A, B\in \wp (S)$, then $A$ is said to be \emph{roughly included} in $B$ $(A\sqsubseteq B)$ if and only if $A^l \subseteq B^l$ and $A^u\subseteq B^u$. $A$ is roughly equal to $B$ ($A\approx B$) if and only if $A\sqsubseteq B$ and $B\sqsubseteq A$. The positive, negative and boundary region determined by a subset $A$ are respectively $A^l$, $A^{uc}$ and $A^{u}\setminus A^l$ respectively.

Boolean algebra with approximation operators constitutes a semantics for classical rough sets (though not satisfactory). This continues to be true even when $R$ in the approximation space is replaced by any binary relation. More generally it is possible to replace $\wp (S)$ by some set with a part-hood relation and some approximation operators defined on it \cite{am240}. The associated semantic domain in the sense of a collection of restrictions on possible
objects, predicates, constants, functions and low level operations on those is 
referred to as the classical semantic domain for general rough sets. In contrast, the semantic domain associated with sets of roughly equivalent or relatively indiscernible objects with
respect to this domain is a \emph{rough semantic domain}. Other semantic domains, including hybrid semantic domains, can be generated and have been used often (for example in choice-inclusive semantics \cite{am99}). Many times such use is not explicitly mentioned or is referred to as \emph{meta levels} in the AI literature.  

\subsection{Granules and Granulations}

Granules or information granules are often the minimal discernible concepts that can be used to construct all relatively crisp complex concepts in a vague reasoning context. Such constructions typically depend on substantial assumptions made by the framework employed in question \cite{am5559,am240,am501,am9411,am9969,tyl}. Major granular computing approaches can be classified into 
\begin{itemize}
\item {Primitive Granular Computing Paradigm: PGCP (see \cite{am501})}
\item {Classical Granular Computing Paradigms: CGCP including precision based approaches, GrC model approach \cite{tyl}, and definite object approaches \cite{yzm2012,hmy2019}.}
\item {Axiomatic Granular Computing Paradigm: AGCP due to the present author}
\end{itemize}

In the present author's axiomatic approach to granularity \cite{am5559,am240,am9114,am9411,am501,am3930,am3600}, fundamental ideas of non-intrusive data analysis have been critically examined and methods for reducing contamination of data (through external assumptions) have been proposed. The need to avoid over-simplistic constructs like rough membership and inclusion functions have been stressed in the approach by her.

\subsection{Algebraic Concepts}

A \emph{semiring} is an algebra of the form $A = \left\langle \underline{A},\, +, \, \cdot,\, 0  \right\rangle$ with $\underline{A}$ being a set,   $ + $ being a commutative monoidal operation with unit element $0$, and $ \cdot$ being an associative operation that satisfies the following distributivity conditions:
\begin{align*}
(\forall a, b, c) a\cdot (b+c) = (a\cdot b) + (a\cdot c)    \tag{l-distributivity}\\
(\forall a, b, c) (b+c)\cdot a = (b\cdot a) + (c\cdot a)    \tag{r-distributivity}
\end{align*}
A \emph{hemiring} is an algebra of the same type without the unit element $0$.

For basics of partial algebras, the reader is referred to \cite{bu,lj}.
\begin{definition}
A \emph{partial algebra} $P$ is a tuple of the form \[\left\langle\underline{P},\,f_{1},\,f_{2},\,\ldots ,\, f_{n}, (r_{1},\,\ldots ,\,r_{n} )\right\rangle\] with $\underline{P}$ being a set, $f_{i}$'s being partial function symbols of arity $r_{i}$. The interpretation of $f_{i}$ on the set $\underline{P}$ should be denoted by $f_{i}^{\underline{P}}$, but the superscript will be dropped in this paper as the application contexts are simple enough. If predicate symbols enter into the signature, then $P$ is termed a \emph{partial algebraic system}.   
\end{definition}

In this paragraph the terms are not interpreted. For two terms $s,\,t$, $s\,\stackrel{\omega}{=}\,t$ shall mean, if both sides are defined then the two terms are equal (the quantification is implicit). $\stackrel{\omega}{=}$ is the same as the existence equality (also written as $\stackrel{e}{=}$) in the present paper. $s\,\stackrel{\omega ^*}{=}\,t$ shall mean if either side is defined, then the other is and the two sides are equal (the quantification is implicit). Note that the latter equality can be defined in terms of the former as 
\[(s\,\stackrel{\omega}{=}\,s \, \longrightarrow \, s\,\stackrel{\omega}{=} t)\&\,(t\,\stackrel{\omega}{=}\,t \, \longrightarrow \, s\,\stackrel{\omega}{=} t) \]

\subsection{Meaning of Terms}

This list is to help with reading about general rough sets in the frameworks used. Motivations can be found all over the rough set literature (see \cite{am501}).

\begin{itemize}
\item {\textsf{Crisp Object}:  That which has been designated as \emph{crisp} or is an approximation of some other object.}
\item {\textsf{Vague Object}: an object that differs from its approximations.}
\item {\textsf{Discernible Object}: That which is available for computations in a rough semantic domain (in a contamination avoidance perspective). }
\item {\textsf{Rough Object}: Many definitions and representations are possible relative to the context. From the representation point of view these are usually functions of definite or crisp objects.}
\item {\textsf{Definite Object}: An object that is invariant relative to an approximation process. In actual semantics a number of concepts of definiteness is possible. In some approaches, as in \cite{yzm2012,hmy2019}, these are taken as granules. Related theory has a direct connection with closure algebras and operators as indicated in \cite{am501}.}
\end{itemize}

\section{Contamination and Intrusive Analysis}

The concept of contamination is explained abstractly and with a concrete example. Generalized versions of RIFs would be a good choice in the context of the example.

Some of the principles used in this research can be nicely summarized:

Let an interpretation of object $a_i$ in domain $\mathfrak{Z}$ be $b_i$ in domain $\mathfrak{W}$ for each $i\in \{1, 2, \ldots n\}$ ($n$ being a finite positive integer), and $\phi: \mathbb{S} \longmapsto Q$ a partial function from the collection of objects to the rationals, then  
\begin{description}
\item [Contaminated version-1:]{a measure of the form\\ $f(a_1,\ldots a_n) = t( \phi(a_1), \ldots, \phi(a_n))$ (whenever defined) to denote an aspect the state of affairs in $\mathfrak{W}$ relative to $\mathfrak{Z}$}
\item [Less-Contaminated version-1:]{a measure of the form\\ $f(a_1,\ldots a_n) = t( \phi(b_1), \ldots, \phi(b_n))$ (whenever defined) to denote an aspect of the state of affairs in $\mathfrak{W}$ relative to $\mathfrak{Z}$}
\end{description}

But the above is typically hindered by the need to use multiple approximations (like a pair of lower and upper approximations). So $f$ may need to be replaced by other alternative constructs that capture the essence of the generalized reasoning.

Let an interpretation of object $a_i$ in domain $\mathfrak{Z}$ be $b_{i_1}, b_{i_2},\ldots, \, b_{i_r} $ in domain $\mathfrak{W}$ for each $i\in \{1, 2, \ldots n\}$ ($n$ and $r$ being finite positive integers), and $\phi: \mathbb{S} \longmapsto Q$ a partial function from the collection of objects to the rationals, then  
\begin{description}
 \item [Contaminated version-2:]{a measure of the form\\ $f(a_1,\ldots a_n) = t( \phi(a_1), \ldots, \phi(a_n))$ (whenever defined) to denote an aspect of the state of affairs in $\mathfrak{W}$ relative to $\mathfrak{Z}$}
 \item [Less-Contaminated version-2:]{a measure of the form\\ $f(a_1,\ldots a_n) = \xi(t( \phi(b_{1_1}), \ldots, \phi(b_{n_1})),\ldots, t( \phi(b_{1_r}), \ldots, \phi(b_{n_r}))  )$ (whenever defined) to denote an aspect of the state of affairs in the interpretation of $\mathfrak{W}$ relative to $\mathfrak{Z}$. $\xi$ may be a morphism or map $R^r \longmapsto R^k$ for some positive integer $k \leq r$.}
\end{description}

In general, contamination does not have a simple form. The following is a nice but very dense way of expressing it.

\begin{description}
 \item [Contaminated:]{a set of formulas $X\cup V$ is used in domain $\mathfrak{Z}$ to model some phenomena in $\mathfrak{W}$ but of these $V$ is not reasonable in domain $\mathfrak{W}$ because they can be disproved in some instances at least (though $V$ may be apparently intuitively justified in some other cases).}
 \item [Less-Contaminated:]{Avoid using $V$ or its unjustified consequences in the modeling and do it in a domain-aware way.}
\end{description}

\subsection{Soft Aids for Pedestrians}\label{pedn}

While autonomous cars are likely to see some adoption, the related question of helping humans walk better through soft methods remains largely unexplored. \emph{In fact, the latter problem may be far more formidable for AI scientists}. There is some literature on improving the alertness of mobile users during walking. But the problem of designing soft aids to improve human navigation (especially for walking on streets) is hindered by contamination because a pedestrian's perspective is too complex. Further people do not develop uniform skill sets in response to adverse conditions. Aspects of this problem are explained below.

In many unplanned cities and towns, people on foot may share the roads (often in a bad state) with vehicles because of the absence of proper footpaths and driving lanes. For example, there are such roads close to the present author's institute. The alertness level of people and vehicle drivers needs to be much higher than on dedicated pathways. Further they need to do real-time decision-making about their current path. This is determined by factors such as approaching vehicles, pedestrians, condition of road, movements, and other objects on the side of the road among others. People walking on such roads are likely to develop a large number of techniques that are not commonly used by those used to walk only on optimized spacious footpaths. Almost all are likely to move at slower speeds, and stop frequently for oncoming traffic.

Suppose that a device can collect and process information with a number of sensors placed by such a roadside. This can, for example, be done to light a path on the road or to alert pedestrians. Now at least five semantic domains can be associated with the scenario (and in each of these again multiple models can be formally defined):
\begin{itemize}
\item {$\mathfrak{W}$: Semantic Domain relative to a typical pedestrian}
\item {$\mathfrak{S}$: Semantic Domain relative to the AI system}
\item {$\mathfrak{SV}$: Semantic Domain relative to an empowered moving vehicle with sensors}
\item {$\mathfrak{V}$: Semantic Domain relative to a ordinary vehicle }
\item {$\mathfrak{W+}$: Semantic Domain relative to a AI-enhanced pedestrian}
\end{itemize}

Concepts expressed in one semantic domain may not be available elsewhere. Keeping them separate helps because pedestrians typically use much smaller subsets of information in real time to decide their next moves. If the road condition is also bad, then suboptimal decisions may be taken. To be more specific, from all available data it may be deducible through a rough set approach that a pedestrian's left feet should be placed on some jagged region. But the person's decision may be to place it on a shallower puddle of water instead of a deeper one, and this may be a good rough approximation in the context. The former solution is contaminated, and the latter is less contaminated. A pedestrian may believe in her solution to a greater extent than on the other, and the superiority of her own set of moves (that are unlikely to be captured by the modeling mechanism). Further, actions based on unjustified solutions are likely to induce anxiety that can have bad consequences. 

\emph{This example shows that reducing contamination is important in contexts that warrant it}.
It is of course, possible to work with a single semantic domain instead, compare multiple models for a purpose, and get intractable solutions.

\subsection{Contamination and Granular Operator Spaces}\label{cav}

Rough sets may be a viewed as a plural discipline in which diverse theoretical and practical approaches are used. In the former, it is necessary to select a few of the following assumptions or actionable items :
\begin{itemize}
\item[RT1]{Decide to proceed from concrete information about approximations derived from data tables to the abstract.}
\item[RT2]{Decide to proceed from concrete or abstract information about approximations.}
\item[RT3]{Assume that semantic models should concern all available information (in the classical semantic domain).}
\item[RT4]{Assume that a negation operation exists and duality between lower and upper approximations holds.}
\item[RT5]{Assume that semantic models should concern only those information available in the domain of rough reasoning (rough semantic domain).}
\item[RT6]{Assume that semantic models should concern only those information available in a domain of reasoning.}
\item[GT1]{Decide to use the axiomatic framework of granularity(AGCP) due to the present author \cite{am501,am240}.}
\item[GT01]{Decide to use the restricted frameworks of granularity afforded by cover based rough sets (an interpretation of the GrC model).} 
\item[GT02]{Decide to use the precision based framework (CGCP) of granularity.}
\item[GT03]{Decide to use the idea that all definite objects and their aggregation are granules.}
\item[GT2]{Decide on concrete or abstract definition of\\ approximations. Concrete approximations being those approximations representable in terms of attributes, points, other operations and collectivizing strategies available in the context, while abstract approximations are approximations satisfying a set of properties alone. Algebraic approximations are typically a subset of concrete approximations.}
\end{itemize}

The combinations \textsf{RT1, RT3, RT4} and restriction to mostly pointwise approximations leads to modal  and other logics. In this granularity followed is \textsf{GT01}, though it does not always play a direct role as far as representation of approximations are concerned. In cover based rough sets (see \cite{cptrs19,pp2018,ppm2} in particular), a number of representations corresponding to formulas do happen. Most semantics in these contexts correspond to the classical domain. \textsf{GT03} is an approach suggested in \cite{yzm2012,hmy2019} -- the idea that all granules should be definite objects has a long history, but requiring the converse falls within debatable territory.

In a number of papers including \cite{cc5,gcd2018,cd3,gc2018}, the combination \textsf{RT2, RT3, RT4} has been studied without additional assumptions about granularity. The entire reasoning process modeled is in a classical domain.

A substantial part of the present author's work (especially \cite{am240,am501,am6900,am9969,am9204,am3600,am3930,am1800,am6999}) falls under the assumptions \textsf{RT2, RT5, GT1, GT2}. In her work, the axiomatic approach to granularity, the abstract frameworks of granular operator spaces and a minimalist approach corresponding to contamination avoidance have played a central role. Granular operator spaces and variants are essentially one set theoretic implementation of the philosophical considerations underlying the general mereology-based framework of rough Y-systems  \cite{am240}. All variants of granular operators spaces and higher versions are intended to capture scenarios involving some collections of sets or objects, restricted by parthood, other orders, and abstract approximations restricted by granulations. Higher versions can be represented as partial or total algebras, while lower versions cannot be so represented. Despite the difference, they are intended to model similar scenarios. \emph{So effectively, all granular operator spaces can be related to Pre-GGS considered in Section \ref{varggs}}.

Granular operator spaces and variants (specifically high granular operator spaces) have the following features:
\begin{itemize}
\item {They adhere to the weak definitions of granularity as per the axiomatic granular approach,}
\item {They do not assume a negation operation,}
\item {Their universe maybe a collection of rough objects (in some sense), or a mix of rough and non rough objects or even a collection of all objects involved,}
\item {The sense of parthood between objects is assumed to be distinct from other order relations,}
\item {Realistic partial aggregation and commonality operations are permitted, and}
\item {Numeric simplified measures are not assumed in general.}
\end{itemize}
These features are motivated by properties satisfied by models in real reasoning contexts, and helps in avoiding contamination to a substantial extent. In addition, it should be noted that the problem of avoiding contamination requires pro-active methods of solution. The other two theoretical approaches mentioned do not have all the features mentioned and the application contexts/intents are different. These will be taken up in a separate paper. 

In practical applications to feature selection and computing reducts, suitability of a procedure used depends on the application context itself. As mentioned earlier, a procedure $X$ may be less contaminated than $Z$ in general, but its utility value is dictated by feasibility and cost of the computations associated and subjective relevance of $X$.

\section{Variants of Granular Operator Spaces}\label{varggs}

Granular operator spaces, variants thereof and related partial algebras are abstract constructs that can express 
\begin{description}
\item[Mer] {part of relations between objects,}
\item[App] {approximations of objects,}
\item[Ord] {generalized orders between objects, }
\item[Gra] {granular properties of approximations, and}
\item[Oth] {other rough-set theoretic properties}
\end{description}

The main advantage over other abstractions is that granules are directly part of the discourse and the distinction between \textsf{Mer} and \textsf{Ord} is maintained. Further the conditions on approximations are minimalist. As many as six variants of such spaces have been defined by the present author -- these can be viewed as special cases of a set theoretic and a relation-theoretic abstraction with abstract operations from a category theory perspective. As proved in \cite{am5550,am5559}, higher order versions are partial algebraic systems, the \emph{space} in the terminology is because of mathematical usage conventions. 

Motivations for the present approach relate to issues of simplifying rough Y-systems (RYS) \cite{am240} to purely set theoretic contexts which in turn were motivated by the need to accommodate granulations and simultaneously generalize abstract approaches to rough sets \cite{it2,cd3,gc2018} without superfluous assumptions. But over time, the level of abstraction has evolved to cover more ground beyond general rough and fuzzy set theories. In the literature on mereology \cite{av,vie,ur,rgac15,am3930,seibtj2015}, it is argued that most ideas of binary \emph{part of} relations in human reasoning are at least anti-symmetric and reflexive.  \emph{A major reason for not requiring transitivity of the parthood relation is because of the functional reasons that lead to its failure} (see \cite{seibtj2015}), and to accommodate \emph{apparent parthood} \cite{am9969}. In the context of approximate reasoning interjected with subjective or pseudo-quantitative degrees, transitivity is again not common. The role of such parthoods in higher order approaches are distinctly different from theirs in lower order approaches -- specifically, general approximation spaces of the form $S$ mentioned above with $R$ being a parthood relation are also of interest.

In a \emph{high general granular operator space} (\textsf{GGS}), defined below, aggregation and co-aggregation operations ($\vee, \,\wedge$) are conceptually separated from the binary parthood $\pc$), and a basic partial order relation ($\leq$). Parthood is assumed to be reflexive and antisymmetric. It may satisfy additional generalized transitivity conditions in many contexts. Real-life information processing often involves many non-evaluated instances of aggregations (disjunctions), co-aggregation (conjunctions) and implications because of laziness or supporting meta data or for other reasons  -- this justifies the use of partial operations. Specific versions of a \textsf{GGS} and granular operator spaces have been studied in \cite{am501} by the present author for handling a very large spectrum of rough set contexts. \textsf{GGS} has the ability to handle adaptive situations as in \cite{skaj2016,skajsd2016} through special morphisms -- this is again harder to express without partial operations.  

The underlying set $\underline{\mathbb{S}}$ can be a set of sets of attributes, but this interpretation is not compulsory. In actual practice, \textsf{the set of all attributes in a context need not be known exactly to the reasoning agent constructing the approximations. The element $\top$ may be omitted in these situations or the issue can be managed through restrictions on the granulation}. While the elements of $\mathbb{S}$ can have many meanings,

\begin{definition}\label{gfsg}

A \emph{High General Granular Operator Space} (\textsf{GGS}) $\mathbb{S}$ is a partial algebraic system  of the form $\mathbb{S} \, =\, \left\langle \underline{\mathbb{S}}, \gamma, l , u, \pc, \leq , \vee,  \wedge, \bot, \top \right\rangle$ with $\underline{\mathbb{S}}$ being a set, $\gamma$ being a unary predicate that determines $\mathcal{G}$ (by the condition $\gamma x$ if and only if $x\in \mathcal{G}$) 
an \emph{admissible granulation}(defined below) for $\mathbb{S}$ and $l, u$ being operators $:\underline{\mathbb{S}}\longmapsto \underline{\mathbb{S}}$ satisfying the following ($\underline{\mathbb{S}}$ is replaced with $\mathbb{S}$ if clear from the context. $\vee$ and $\wedge$ are idempotent partial operations and $\pc$ is a binary predicate. Further $\gamma x$ will be replaced by $x \in \mathcal{G}$ for convenience.):

\begin{align*}
(\forall x) \pc xx \tag{PT1}\\
(\forall x, b) (\pc xb \, \&\, \pc bx \longrightarrow x = b) \tag{PT2}\\
(\forall a, b) a\vee b \stackrel{\omega}{=} b\vee a \;;\; (\forall a, b) a\wedge b \stackrel{\omega}{=} b\wedge a \tag{G1}\\
(\forall a, b) (a\vee b) \wedge a \stackrel{\omega}{=} a \; ;\; (\forall a, b) (a\wedge b) \vee a \stackrel{\omega}{=} a \tag{G2}\\
(\forall a, b, c) (a\wedge b) \vee c \stackrel{\omega}{=} (a\vee c) \wedge (b\vee c) \tag{G3}\\
(\forall a, b, c) (a\vee b) \wedge c \stackrel{\omega}{=} (a\wedge c) \vee  (b\wedge c) \tag{G4}\\
(\forall a, b) (a\leq b \leftrightarrow a\vee b = b \,\leftrightarrow\, a\wedge b = a  ) \tag{G5}\\
(\forall a \in \mathbb{S})\,  \pc a^l  a\,\&\,a^{ll}\, =\,a^l \,\&\, \pc a^{u}  a^{uu}  \tag{UL1}\\
(\forall a, b \in \mathbb{S}) (\pc a b \longrightarrow \pc a^l b^l \,\&\,\pc a^u  b^u) \tag{UL2}\\
\bot^l\, =\, \bot \,\&\, \bot^u\, =\, \bot \,\&\, \pc \top^{l} \top \,\&\,  \pc \top^{u} \top  \tag{UL3}\\
(\forall a \in \mathbb{S})\, \pc \bot a \,\&\, \pc a \top    \tag{TB}
\end{align*}

Let $\pp$ stand for proper parthood, defined via $\pp ab$ if and only if $\pc ab \,\&\,\neg \pc ba$). A granulation is said to be admissible if there exists a term operation $t$ formed from the weak lattice operations such that the following three conditions hold:
\begin{align*}
(\forall x \exists
x_{1},\ldots x_{r}\in \mathcal{G})\, t(x_{1},\,x_{2}, \ldots \,x_{r})=x^{l} \\
\tag{Weak RA, WRA} \mathrm{and}\: (\forall x)\,(\exists
x_{1},\,\ldots\,x_{r}\in \mathcal{G})\,t(x_{1},\,x_{2}, \ldots \,x_{r}) =
x^{u},\\
\tag{Lower Stability, LS}{(\forall a \in
\mathcal{G})(\forall {x\in \underline{\mathbb{S}}) })\, ( \pc ax\,\longrightarrow\, \pc ax^{l}),}\\
\tag{Full Underlap, FU}{(\forall
x,\,a \in\mathcal{G})(\exists
z\in \underline{\mathbb{S}}) )\, \pp xz,\,\&\,\pp az\,\&\,z^{l}\, =\,z^{u}\, =\,z,}
\end{align*}
\end{definition}
\emph{The conditions defining admissible granulations mean that every approximation is somehow representable by granules in a algebraic way, that every granule coincides with its lower approximation (granules are lower definite), and that all pairs of distinct granules are part of definite objects (those that coincide with their own lower and upper approximations).} Special cases of the above are defined next. For details, the reader is referred to \cite{am5586,am501}.

\begin{definition}
\begin{itemize}
\item {In a \textsf{GGS}, if the parthood is defined by $\pc ab$ if and only if $a \leq b$ then the \textsf{GGS} is said to be a \emph{high granular operator space} \textsf{GS}.}
\item {A \emph{higher granular operator space} (\textsf{HGOS}) $\mathbb{S}$ is a \textsf{GS} in which the lattice operations are total.}
\item {In a higher granular operator space, if the lattice operations are set theoretic union and intersection, then the \textsf{HGOS} will be said to be a \emph{set HGOS}.}
\end{itemize}
\end{definition}

\begin{definition}
In the context of Def. \ref{gfsg} if additional lower and upper approximation operations are present in the signature, then  the system will be referred to as an \emph{enhanced} \textsf{GGS} (\textsf{EGGS}. 
\end{definition}

Under certain conditions, partial or total groupoid operations can correspond to binary relations on a set. Related algebraic results are used by the present author (see \cite{am5550}) to rewrite the GGS as partial algebras \cite{am5550}. These results can be used to improve the axiomatic frameworks used.

\subsection{Examples of GGS}

In general, \textsf{GGS} cannot be used to formalize approaches to rough sets that are based on non granular approximations, and so it is not a framework for everything. 
 
A general definition of \emph{point-wise approximations} can be proposed in Second Order Predicate Logic(SOPL) (or alternatively, in a fixed language) based on the following loose SOPL version :  If $S$ is an algebraic system of type $\tau$ and $\nu: S \longmapsto \wp (S)$ is a neighborhood map on the universe $S$, then a \emph{point-wise approximation} $*$ of a subset $X\subseteq S$ is a self-map on $\wp(S)$ that is definable in the form: 
\begin{equation}
X^*\, =\,\{x: \, x\in H\subseteq S \, \& \, \Phi(\nu(x), X) \} 
\end{equation}
for some formula $\Phi(A, B)$ with $A, B \in \wp(S)$. In classical rough sets point-wise approximations lead to a granular semantics, but in other cases they do not in general. 

The full generality implicit in a \textsf{GGS} is not usually required for expressing most granular rough set approaches. So in the following example - this aspect is targeted.

\begin{example}
Suppose the problem at hand is to represent the knowledge of a specialist in automobile engineering and production lines in relation to a database of cars, car parts, calibrated motion videos of cars and performance statistics. The database is known to include a number of experimental car models and some sets of cars have model names, or engines or other crucial characteristics associated. 

Let $\underline{\mathbb{S}}$ be the set of cars, some subsets of cars, sets of internal parts and components of many cars. $\mathcal{G}$ be the set of internal parts and components of many cars. Further let 
\begin{itemize}
\item {$\pc a b$ express the relation that $a$ is a possible component of $b$ or that $a$ belongs to the set of cars indicated by $b$ or that   }
\item {$a \leq b$ indicate that $b$ is a better car than $a$ relative to a certain fixed set of features,}
\item {$a^l$ indicate the closest standard car model whose features are all included in $a$ or set of components that are included in $a$, }
\item {$a^u $ indicate the closest standard car model whose features are all included by $a$ or fusion of set of components that include $a$}
\item {$\vee$, $\wedge$ can be defined as partial operations, while $\bot$ and $\top$ can be specified in terms of attributes. }
\end{itemize}
Under the conditions, \[\mathbb{S} \, =\, \left\langle \underline{\mathbb{S}}, \mathcal{G}, l , u, \pc, \leq , \vee,  \wedge, \bot, \top \right\rangle\] forms a \textsf{GGS}.

Suppose that the specialist has updated her knowledge over time, then this transformation can be expressed with the help of morphisms from a \textsf{GGS} to itself. 
\end{example}

\subsubsection*{Abstract Example}

Let \begin{multline*}
\underline{\mathbb{S}} = \{\emptyset , \{a, b, c, e\}, \{a, b, e\}, \{b, c, e\}, \{a, b\}, \{b, e\}, \{b, c\}, \{a\}, \{e\} \}     
\end{multline*}

The  elements of $\underline{\mathbb{S}}$ can be read as collections of attributes. Further, let

\begin{align*}
\top =\{a, b, c, e\}   \tag{Top}\\
\bot = \emptyset   \tag{Bottom}\\
\mathcal{G} = \{\{b, e\}, \{b, c\}, \{a\}, \{e\}\}  \tag{Granulation}
\end{align*}

Define the predicate $\pc$, and the operations $\vee, \wedge, l, u$ as follows (for any $x, w\in \underline{\mathbb{S}} $):
\begin{multline*} \pc = \{(\bot, x), (x,\top), (\{a, b\},\{a, b, e\}),\\(\{b, e\},\{b, e, c\}),(\{a\}, \{a,b\}), (\{b,c\}, \{b,c,e\}),(\{x\}, \{x\}),(\{e\}, \{b,e\})\}\end{multline*}
\begin{align}
(\forall x, w\in \underline{\mathbb{S}\setminus\{\bot\}}) x\vee w := x\cup w \text{ if defined}   \tag{wag}\\
(\forall x, w\in \underline{\mathbb{S}\setminus\{\bot\}}) x \wedge w :=  x \cap w \text{ if defined}   \tag{wco}\\
x^l := \bigcup \{z: \, z\in \mathcal{G}\, \&\,\pc z x \}   \tag{lower}\\
x^u := \bigcup \{z: \, z\in \mathcal{G}\, \&\,z\cap x \neq \emptyset \}   \tag{upper}
\end{align}

The definition of the upper approximation is actually external to the partial algebra (as set intersection is not defined in it), but the result is a weak aggregation ($vee$) of granules. 
$\leq$ is representable as in Figure \ref{leqr}. If $w\leq x$ for some $x$ and $w$, then an arrow is drawn from $x$ to $w$. Note that $\{b, e\} \wedge \{b, c\}$ is not defined (for example), and there are no arrows to $\bot$.
\begin{center}
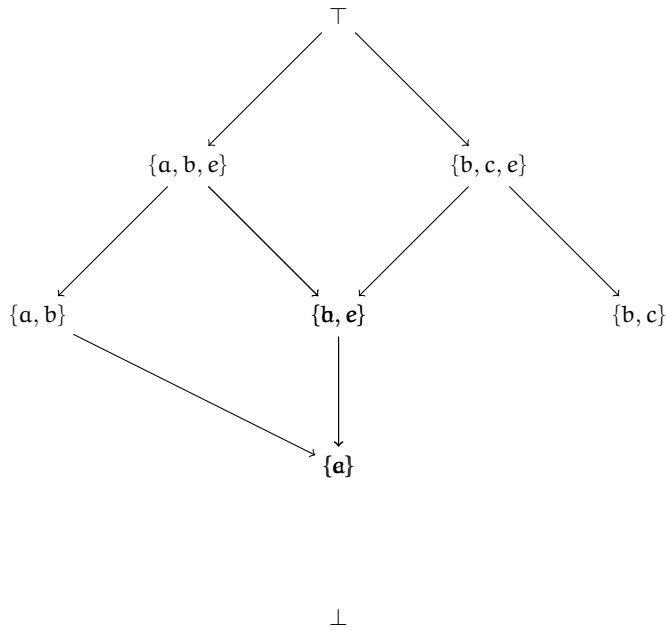
\begin{figure}[hbt]
\begin{tikzpicture}[node distance=2cm, auto]
\node (T) {$\top$};
\node (O) [below of=T] {};
\node (A3) [left of=O] {$\{a, b, e\}$};
\node (B3) [right of=O] {$\{b,c, e\}$};
\node (H) [below of=A3] {};
\node (A21) [left of=H] {$\{a, b\}$};
\node (A22) [right of=H] {$\{a, e\}$};
\node (G) [below of=B3] {};
\node (B21) [left of=G] {$\{b, e\}$};
\node (B22) [right of=G] {$\{b, c\}$};
\node (A1) [below of=A22] {$\{a\}$};
\node (E1) [below of=B21] {$\{e\}$};
\node (F) [below of=E1] {$\bot$};
\draw[->,font=\scriptsize] (T) to node {}(A3);
\draw[->,font=\scriptsize] (T) to node {}(B3);
\draw[->,font=\scriptsize] (A3) to node {}(A21);
\draw[->,font=\scriptsize] (A3) to node {}(A22);
\draw[->,font=\scriptsize] (A3) to node {}(B21);
\draw[->,font=\scriptsize] (B3) to node {}(B21);
\draw[->,font=\scriptsize] (B3) to node {}(B22);
\draw[->,font=\scriptsize] (A21) to node {}(A1);
\draw[->,font=\scriptsize] (A22) to node {}(A1);
\draw[->,font=\scriptsize] (A22) to node {}(E1);
\draw[->,font=\scriptsize] (B21) to node {}(E1);
\end{tikzpicture}
\caption{The $\leq$ Relation}
\label{leqr}
\end{figure}
\end{center}

Note that no arrow leads to $\bot$. The essential diagram (omitting $\bot$ and $\top$) for parthood $\pc$ is in Figure \ref{partr}. If $\pc xw$ for some $x$ and $w$, then an arrow is drawn from $x$ to $w$.

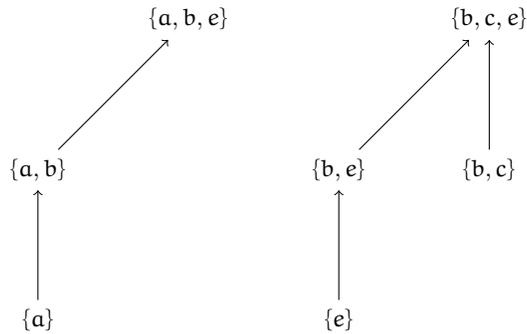
\begin{figure}[hbt]
\begin{center}
\begin{tikzpicture}[node distance=2cm, auto]
\node (T) {};
\node (A3) [left of=T] {$\{a, b, e\}$};
\node (B3) [right of=T] {$\{b,c, e\}$};
\node (H) [below of=A3] {};
\node (A21) [left of=H] {$\{a, b\}$};
\node (B21) [right of=H] {$\{b, e\}$};
\node (B22) [right of=B21] {$\{b, c\}$};
\node (A1) [below of=A21] {$\{a\}$};
\node (E1) [below of=B21] {$\{e\}$};
\draw[<-,font=\scriptsize] (A3) to node {}(A21);
\draw[<-,font=\scriptsize] (A21) to node {}(A1);
\draw[<-,font=\scriptsize] (B3) to node {}(B21);
\draw[<-,font=\scriptsize] (B3) to node {}(B22);
\draw[<-,font=\scriptsize] (B21) to node {}(E1);
\end{tikzpicture}
\caption{The $\pc$ Relation}
\label{partr}
\end{center}
\end{figure}

The approximation table for the situation is computed in Table \ref{appr}

\begin{table}[hbt]
 \centering
\begin{tabular}{ccc}
\toprule
\textsf{$x$} & $x^l$ & $x^u$  \\
\midrule
$\{ a\}$  & $\{ a\}$ & $\{ a\}$ \\
\midrule
$\{ e\}$  & $\{e \}$ & $\{ b, e\}$ \\
\midrule
$\{a, b \}$  & $\{ a\}$ & $\{ a, b, c, e\}$ \\
\midrule
$\{b, c \}$  & $\{b, c \}$ & $\{b, c, e \}$ \\
\midrule
$\{b, e \}$  & $\{b, e \}$ & $\{b, c, e \}$ \\
\midrule
$\{a,b, e \}$  & $\{a, b, e \}$ & $\{a, b, c, e \}$ \\
\midrule
$\{b, c, e \}$  & $\{b, c, e \}$ & $\{b, c, e \}$ \\
\midrule
$\{a, b, c, e \}$  & $\{a, b, c, e \}$ & $\{a, b, c, e \}$ \\
\end{tabular}
\caption{The Approximations}\label{appr}
\end{table}

\subsection{Construction of a HGOS}
 
A \textsf{set HGOS} is intended to capture contexts where all objects are described by sets of attributes with related valuations (that is their properties). So objects can be associated with sets of properties (including labels possibly). A more explicit terminology for the concept, may be \emph{power set derived \textsf{HGOS}}(that captures the intent that subsets of the set of all properties are under consideration here). Admittedly, the construction or specification of such a power set is not necessary. In a \textsf{HGOS}, such set of sets of properties need not be the starting point.
 
The difference between a \textsf{HGOS} and a \textsf{set HGOS} at the practical level can be interpreted at different levels of complexity. Suppose that the properties associated with a familiar object like a cast iron frying pan are known to a person $X$, then it is possible to associate a set of properties with valuations that are sufficient to define it. If all objects in the context are definable to a \emph{sufficient level}, then it would be possible for $X$ to associate a \textsf{set HGOS} (provided the required aspects of approximation and order are specifiable). 

It may not be possible to associate a set of properties with the same frying pan in a number of scenarios. For example, another person may simply be able to assign a label to it, and be unsure about its composition or purpose. Still the person may be able to indicate that another frying pan is an approximation of the original frying pan. In this situation, it is more appropriate to regard the labeled frying pan as an element of a \textsf{HGOS}. 

A nominalist position together with a collectivization property can also lead to \textsf{HGOS} that is not a \textsf{set HGOS}.

\section{General Rough Inclusion Functions}

Intuitively, generalizations of rough inclusion functions are likely to work perfectly when 
\begin{description}
\item[A1]{the contribution of attributes to approximations have uniform weightage across approximations,}
\item[A2]{the contributions of attributes in the construction of approximation can be assigned weights,}
\item[A3]{the functions are robust and sensitive (that is the value of the function does not change much with small deviations of its arguments \cite{skajsd2016} and stable relative to the context),}
\item[A4]{every aggregate of attributes is meaningful, and}
\item[A5]{attributes are independent.}
\end{description}
The ideas of robustness and stability are always relative to a finite number of purposes or use cases in application contexts.

In this section, the different known rough inclusion functions are generalized to \textsf{GGS} of the form $\mathbb{S} \, =\, \left\langle \underline{\mathbb{S}}, \mathcal{G}, l , u, \pc, \vee,  \wedge, \bot, \top \right\rangle$. If $\kap : \underline{\mathbb{S}}^2 \longmapsto [0, 1]$ is a map, consider the conditions,
\begin{align*}
(\forall a)\, \kap (a, a) = 1    \tag{U1}\\
(\forall a, b)(\kap (a, b) = 1 \leftrightarrow   \pc a b) \tag{R1}\\
(\forall a, b, c)(\kap (b, c) = 1 \longrightarrow \kap (a, b) \leq \kap (a, c))   \tag{R2}\\
(\forall a, b, c)(\pc bc \longrightarrow \kap (a, b) \leq \kap (a, c))   \tag{R3}\\
(\forall a, b)(\pc ab \longrightarrow \kap (a, b) = 1)   \tag{R0}\\ 
(\forall a, b)(\kap (a, b) = 1\longrightarrow \pc ab )   \tag{IR0}\\ 
(\forall a) (\pp \bot a\longrightarrow  \kap (a, \bot ) = 0)   \tag{RB}\\
(\forall a, b)(\kap ( a, b)=0 \longrightarrow a \wedge b = \bot)  \tag{R4}\\
(\forall a, b)( a\wedge b = \bot \,\&\, \pp \bot a \longrightarrow \kap (a, b) = 0)   \tag{IR4}\\ 
(\forall a, b)(\kap ( a, b)=0 \,\&\, \pp \bot a  \leftrightarrow a \wedge b = \bot)  \tag{R5}\\
(\forall a, b, c)( \pp \bot a \, \&\, b\vee c = \top\longrightarrow  \kap (a , b) + \kap (a, c) =1)  \tag{R6}
\end{align*}

These are analogous extensions of the definition in \cite{ag2009,ag3}. $rif_3$ is RB, and $rif_{2*}$ is R2 in a set HGOS under the conditions mentioned. Proofs of the next theorem for set HGOS can be deduced from those for wqRIFs in \cite{ag2009}. These carry over to \textsf{HGOS} directly, while the proofs in a \textsf{GS} are not hard.

\begin{theorem}
The following implications between the properties are easy to verify.
\begin{description}
\item[prif1]{If $\mathbb{S}$ satisfies R1, then R3 and R2 are equivalent.}
\item[prif2]{R1 if and only if R0 and IR0 are satisfied.}
\item[prif3]{R0 and R2 imply R3.}
\item[prif4]{IR0 and R3 imply R2.}
\item[prif5]{IR4 implies RB.}
\item[prif6]{IR4 and R4 if and only if R5.}
\item[prif7]{When complementation is well defined in a set HGOS then R0 and R6 imply IR4.}
\item[prif8]{When complementation is well defined in a set HGOS then IR0 and R6 imply R4.}
\item[prif9]{When complementation is well defined in a set HGOS then R1 and R6 imply R5.}
\end{description}
Further both R1 and R0 imply U1.
\end{theorem}
\begin{proof}
Aspects of the proof are illustrated below
\begin{description}
\item[prif1]{Suppose $\pc bc$ for some $b, c\in \mathbb{S}$, then by R1 $\kap (b, c) =1$ and conversely. In R2 and R3 the premise can be interchanged in the conditional implication when R1 holds.}
\item[prif2]{Obvious.}
\item[prif3]{Suppose R0 and R2 hold. If $\pc bc$ for some $b, c\in \mathbb{S}$ then by R0 $\kap (b, c) =1$. So for any $a$ $\kap (a, b) \leq \kap (a, c)$. That is R3 follows from R0 and R2}
\item[prif5]{Substituting $\bot$ for $b$ in IR4 yields RB.}
\item[pref6]{Is obvious.}
\end{description}
\end{proof}

\begin{definition}
By a \emph{general rough inclusion} function (RIF) on a \textsf{GGS} $\mathbb{S}$ shall be a map $\kap : (\mathbb{S})^2 \longmapsto [0, 1]$ that satisfies R1 and R2.
A \emph{general quasi rough inclusion} function (qRIF) will be a map $\kap : (\mathbb{S})^2 \longmapsto [0, 1]$ that satisfies R0 and R2. While a \emph{general weak quasi rough inclusion} function (wqRIF) will be a map $\kap : (\mathbb{S})^2 \longmapsto [0, 1]$ that satisfies R0 and R3. 
\end{definition}

\begin{proposition}
In a \textsf{GS} $\mathbb{S}$, every \textsf{RIF} is a \textsf{qRIF} and every \textsf{qRIF} is a \textsf{wqRIF}.  
\end{proposition}

\paragraph{Other Generalizations:}
The above generalizations can also be done relative to the $\leq$ order of a GGS (instead of $\pc$). This leads to a focus on partial aggregation and commonality operations, that may not satisfy nice properties relative to the approximations. This makes the alternative approach a relatively bad one.

\subsection{Specific Weak Quasi-RIFs}

RIFs and variants thereof are defined over power sets in \cite{ss2010,ag2009}. For rewriting them in the set HGOS way, it is necessary to assume that $\underline{\mathbb{S}} = \wp(\top)$, $\top$ being a finite set, $\bot = \emptyset$,  $\pc = \leq = \subseteq$, $\vee = \cup$ and $\wedge = \cap$. Specifically, the following functions have been studied in \cite{ag2009} and have been used to define concepts of approximation spaces.   

\begin{equation*}\label{k1}
\nu_1(A, B) = \left\lbrace  \begin{array}{ll}
 \dfrac{\# (B)}{\# (A\cup B)} & \text{if } A\cup B\neq \emptyset\\
 1 & \text{otherwise}\\
 \end{array} \right. \tag{K1}                                                                                                             
\end{equation*}

\begin{equation*}\label{k2}
\nu_2(A, B) =  \dfrac{\# (A^c \cup B)}{\# (\top)}  \tag{K2}                                                                                                             
\end{equation*}

If $0 \leq s < t \leq 1$, and $\nu : {\mathbb{S}}^2 \longmapsto [0, 1]$ is a \textsf{RIF}, then let $\nu_{s,t}^{\nu} : {\mathbb{S}}^2 \longmapsto [0, 1]$ be a function defined by 

\begin{equation*}\label{kst}
\nu^{\nu}_{s,t}(A, B) = \left\lbrace  \begin{array}{ll}
0 & \text{ if } \nu(A, B) \leq s \\
 \dfrac{\nu (A. B) - s}{t - s} & \text{ if } s < \nu(A, B) < t,\\
 1 & \text{ if } \nu(A, B ) \geq t\\
 \end{array} \right. \tag{Kst}                                                                                                             
\end{equation*}

\begin{proposition}{\cite{ag2009}} 
In general, $\nu^{\nu}_{s,t}$ is a weak quasi \textsf{RIF} and $\nu^{\nu}_{s,1}$ is a quasi \textsf{RIF}. 
\end{proposition}

\subsection{Soft Aids for Pedestrians-2}

In the context of the complicated problem of designing soft aids for pedestrian navigation of subsection \ref{pedn}, a number of learning methods can be used to suggest possible strategies. A pedestrian can be assumed to be in a truncated ellipsoid (with distinguished areas at the base) at any time. The best future position of the person can be expected to be selected from a finite number of possible positions. This can be decided on the basis of the value of weak rough inclusion of one ellipsoid on another provided relevant factors are taken into account. But the computation relative to the pedestrian should be through the action of $\sharp$ on the wqRIF to reduce contamination. While this is not a perfect proposal from the perspective of contamination, it may be useful in practice.

\subsection{Generalized RIFs with Other Measures}

RIFs have been related to a number of numeric measures such as quality of classification\cite{zpb} , variable precision rough sets \cite{zw,ss1,yec2017}, accuracy degree of approximation \cite{zpb}, degrees of closeness \cite{lp2011}, dependence degree of a set of attributes on another \cite{zpb}, dependency degree of a decision set with respect to an attribute set and others. Rough membership functions are usually not related to rough inclusion functions in the literature. In most of these cases, the RIFs involve possibly non-crisp and non-definite objects. Some of the connections are mentioned in \cite{jzd2003}. Here the granularity aspect is discussed in brief.

Variants of RIFs have also been used in reduct computation for contexts involving non granular rough approximations (see \cite{chen2014,wu2002,skg1994}). Limitations of mass and plausibility functions are also mentioned in the context.

\begin{theorem}
If $\mathbb{S}$ is a \textsf{set HGOS}, then 
\begin{itemize}
\item {The accuracy degree of approximation of an element $x$ is \[\alpha(x)  = \dfrac{\#(x^l)}{\#(x^u)} = \nu(x^u, x^l).\] }
\item {The classical rough inclusion degree $\nu(a, b)$ defined by Equation. \ref{rif0} is not a function of crisp objects. The degree of misclassification is $\mu(a, b) = 1 - \nu(a, b)$. It coincides with $\nu(a, b^c)$ whenever $\mathbb{S}$ is closed under complements. }
\item {Relative to a partition $\mathcal{S}$ satisfying $\bigcup \mathcal{S} = \bigcup \mathbb{S}$, or even relative to the granulation $\mathcal{G}$, it is possible to define generalized VPRS approximations of any $X\in \mathbb{S}$ for a pair of parameters $0 < \alpha \leq \beta < 1$ as follows:
\begin{align*}
X^{l_v} = \bigcup \{h: \, h\in \mathcal{G}\, \&\, \nu(h, X) > \beta\} \\
X^{u_v} = \bigcup \{h: \, h\in \mathcal{G}\, \&\, \nu(h, X) > \alpha\}
\end{align*}
These approximations clearly depend on granules and the original set.}
\end{itemize}
\end{theorem}

The latter definition is fixed next:

\begin{definition}
Relative to the granulation $\mathcal{G}$, it is possible to define \emph{fixed generalized VPRS approximations} of any $X\in \mathbb{S}$ for a pair of parameters $0 < \alpha \leq \beta < 1$ as follows:
\begin{align*}
X^{l_v} = \bigcup \{h: \, h\in \mathcal{G}\, \&\, \nu(h, X^l) > \beta\} \\
X^{u_v} = \bigcup \{h: \, h\in \mathcal{G}\, \&\, \nu(h, X^l) > \alpha\}
\end{align*}
These approximations clearly depend on granules or approximations. 
\end{definition}

In the context of \textsf{set HGOS}, if generalized RIFs are used then they can be shown to generate few generalized or very different relationships. While it is possible to define a number of generalizations on the theme, the most interesting ones are those of relative comparison between different pairs of approximations.

\section{Algebraic Systems of wqRIFs}

Generalized RIFs of all types satisfy a number of algebraic properties (many of these have not been previously identified or studied in the literature). These are explored and the most useful are identified and characterized in this section. Contamination reduction can also be achieved to an extent through an algebraic perspective.

Let $wqRIF(\mathbb{S})$ be the set of all wqRIFs on $\mathbb{S}$. Consider the following set of definitions:

\begin{definition}\label{operwqr}
For any $f, g\in wqRIF(\mathbb{S})$, $a, b\in \mathbb{S}$, and $\alpha \in [0,1]$, define
\begin{align*}
(f\otimes g) (a, b):= f(a, b) \cdot g(a, b)   \tag{Product}\\
(f\oplus_{\alpha} g) (a, b):= \alpha f(a, b) + (1-\alpha) g(a, b) \tag{$\alpha$-Sum}\\
(\sharp f) (a, b):= f(a^l, b^l)  \tag{l-Decontamination}\\
(\flat f) (a, b):= f(a^u, b^u)  \tag{u-Decontamination} \\
f\preceq g \text{ if and only if } (\forall a, b) f(a, b)\leq f(a, b)   \tag{Order}
\end{align*}
\end{definition}

\begin{definition}
In a GGS $\mathbb{S}$, for any $f\in wqRIF(\mathbb{S})$ and $a, b\in \mathbb{S}$, define $\varsigma$ by
\begin{equation*}
(\varsigma f) (a, b):= \left\{
\begin{array}{ll}
\max \{f(w, b^l): \,w\in \mathcal{G} \,\&\, \pc wa\}  & \text{if } (\exists!^{\geq 1} w\in \mathcal{G}) \pc wa\\
 1 & \text{Otherwise } \\
\end{array}\right .
 \tag{Granular Sum} 
\end{equation*}

\end{definition}

\begin{theorem}
\begin{itemize}
\item {All the operations and predicates in Definition \ref{operwqr} and $\varsigma$ are well- defined over a GGS $\mathbb{S}$.}
\item {All the operations and predicates in Definition \ref{operwqr} are well-defined over an abstract approximation system $\mathbb{S}$.}
\end{itemize}
\end{theorem}
\begin{proof}
\begin{itemize}
\item {It is necessary to prove that $f\otimes h$ is in $wqRIF(\mathbb{S})$ for any $f, h\in wqRIF(\mathbb{S})$. If for any $a, b\in \mathbb{S}$ $\pc ab$ holds, then $(f\otimes h)(a, b) = f(a, b) \cdot h(a, b) = 1$. This verifies R0. Again for the same $a, b$ and a $c\in \mathbb{S}$, $f(a, c)\leq f(b,c)$ and $h(a,c)\leq h(b,c)$ hold. So $(f(a,c) \cdot h(a,c))\leq (f(b, c)\cdot h(b, c))$ follows. This means R3 is satisfied by $f\otimes h$.  }
\item {Let $\alpha \in [0, 1]$, $f, h\in wqRIF(\mathbb{S})$, and $a, b\in \mathbb{S}$ be such that $\pc ab$ holds. Then $(f\oplus_{\alpha} h)(a, b) = \alpha f(a, b) + (1-\alpha) h(a, b) = \alpha +(1 - \alpha) = 1 $. This verifies R0. Again for the same $a, b$ and a $c\in \mathbb{S}$, $f(a, c)\leq f(b,c)$ and $h(a,c)\leq h(b,c)$ hold. From this $\alpha f(a,c)\leq \alpha f(b, c)$ and $(1 - \alpha) h(a, c)\leq (1 - \alpha) h(b, c)$ follows. Adding the two inequalities yields  $\alpha f(a, c) + (1 - \alpha) h(a, c)\leq \alpha f(b, c) + \leq (1 - \alpha) h(b. c)$. That is $ f\oplus_{\alpha} h (a, c) \leq f\oplus_{\alpha} h (b, c) $.}
\item {For any $a, b\in \mathbb{S}$ $\pc ab$ implies $\pc a^l b^l$. Therefore, $\pc a b $ implies $\sharp f (a, b) = f(a^l, b^l) = 1$. This verifies R0. Again for the same $a$ and $b$, and a $c\in \mathbb{S}$, $f(a, c)\leq f(b, c)$. $\sharp f (a, c) = f(a^l, c^l) \leq f(b^l, c^l)$ holds because $\pc a^l b^l$ follows from $\pc ab$. This completes the verification of R3.}
\item {The well-definedness of $\flat$ is similar to that of $\sharp$. }
\item {The relation $\preceq$ is obviously a subset of $(wqRIF(\mathbb{S}))^2$.}
\item {For verifying R0 and R3 for $\varsigma $, note that if $\pc ab$ for any $a, b\in \mathbb{S}$ and $\pc t a$ for a granule $t$, then $\pc t b^l$ holds in a GGS. Therefore $\pc ab$ implies $(\varsigma f)(a, b) =1 $. Further if $c\in \mathbb{S}$, and $\pc t c$ for a granule $t$ and $\pc z a^l $, implies $\pc z b^l$. This yields the inequality $(\varsigma f)(c, a) \leq (\varsigma f) (c, b)$.}
\end{itemize}

All parts of the above proof except for the portion relating to $\varsigma$ apply to abstract approximation systems.
  
\end{proof}

\begin{proposition}
The operation $\top : (\mathbb{S})^2 \longmapsto [0, 1]$ defined as below is a weak quasiRIF:
\[(\forall a, b) \top (a, b) = 1\] 
\end{proposition}

\begin{definition}\label{wqrifa}
By a $\mathbb{A}$-\emph{wqRIF algebra} $W$ over a GGS $\mathbb{S}$ will be an algebraic system of the form  
(with $\underline{W} = wqRIF(\mathbb{S})$, and $\alpha \in \mathbb{A} \subseteq Q\cap [0, 1]$ - the set of rationals in $[0, 1]$)
\[W = \left\langle \underline{W}, \preceq, \otimes, \{\oplus_{\alpha}\}, \flat, \sharp, \varsigma, \top   \right\rangle\] with the operations, predicates and distinguished elements defined as above. If $\mathbb{A} = Q\cap [0, 1]$, then the algebra will be referred to as a \emph{wqRIF algebra}.
\end{definition}

\begin{proposition}
If $\mathbb{S}$ is a finite GGS, then any $f\in wqRIF(\mathbb{S})$ satisfies 
\[\Im (f) \subset [0, b] \cup \{1\}  \] for a fixed $b <1$ ($\Im(f)$ being the image of $f$).
\end{proposition}

\begin{proof}
The proof of the result is obvious. 
\end{proof}

\begin{theorem}\label{wqalg}
In a wqRIF algebra over a GGS (as in Def. \ref{wqrifa}) all of the following  hold:
\begin{align*}
 (\forall f, h)\, f \otimes h = h\otimes f \tag{Comm}\\
 (\forall f, h, t)\, f\otimes (h\otimes t) = (f \otimes h)\otimes t \tag{Assoc} \\
 (\forall f) \, f \otimes \top = f \tag{Identity} \\
 (\forall f) f \oplus_\alpha f = f \tag{Idempotence}\\
(\forall f, t, h \in W)(\forall \alpha\in [0, 1])\, f \otimes (t \oplus_{\alpha} h) = (f \otimes t) \oplus_{\alpha} \tag{Distributivity}
 \end{align*}
\end{theorem}
\begin{proof}
\begin{itemize}
\item {For any $a, b\in \mathbb{S}$ and $f, h\in W$, $(f\otimes h)(a, b) = f(a, b)\cdot h(a, b)$. But $f(a, b)\cdot h(a, b) = h(a, b) \cdot f(a, b) = (h\otimes f)(a, b)$. So commutativity holds. }
\item {For any $a, b\in \mathbb{S}$ and $f, h, t\in W$, $(f\otimes (h\otimes t))(a, b) = f(a, b)\cdot (h(a, b)\cdot t(a, b)) = (f(a, b)\cdot h(a, b))\cdot t(a, b)$ $= ((f\otimes h)\otimes t)(a, b)$.}
\item {For any $a, b\in \mathbb{S}$ and $f \in W$, $(f\otimes \top)(a, b) = f(a, b) \cdot 1  = f(a, b)$. So $f\otimes \top = f$.}
\item {Idempotence for any $alpha$ can be verified directly.}
\item {For any $a, b\in \mathbb{S}$, $\alpha \in [0, 1]$, and $f, h, t\in W$, 
\begin{itemize}
\item {$f\otimes (t\oplus_{\alpha} h)(a, b) = f(a, b)(t(a, b) \oplus_{\alpha} h(a, b) ) =$}
\item {$= f(a, b) (\alpha t(a, b) + (1-\alpha)h(a, b)) =$ }
\item {$= \alpha (f(a, b) t(a, b)) + (1 - \alpha) (f(a, b)h(a, b))$ $ = \alpha (f\otimes t)(a, b) + (1-\alpha) (f\otimes h)(a, b)$.}
\end{itemize}
Clearly that is $((f \otimes t) \oplus_{\alpha} (f\otimes h))(a, b)$. This proves the distributive property.}
\end{itemize}
 
\end{proof}

\begin{theorem}
In a wqRIF algebra over a GGS (as in Def. \ref{wqrifa}) all of the following  hold:
\begin{align*}
 (\forall f, h, f', h')(f\preceq h\, \&\, f'\preceq g' \longrightarrow f\otimes f' \preceq h \otimes h') \tag{Order-1} \\
 (\forall f, h, f', h')(\forall \alpha)(f\preceq h\, \&\, f'\preceq h' \longrightarrow f\oplus_{\alpha} f' \preceq h \oplus_{\alpha} h') \tag{Order-2} \\
(\forall f) \, f \preceq \top = f \tag{Top}\\
(\forall f\in W)(\forall a, b)(\pc aa^l \longrightarrow (\sharp f)(a, b) \leq f(a, b) )   \tag{weak sharp comp}\\
(\forall f\in W)(\forall a, b) (\pc a^ua \longrightarrow  f(a, b) \leq (\flat f)(a, b))   \tag{weak flat comp}\\
(\forall f\in W)(\forall a, b) (\pc ab \longrightarrow  \varsigma f (a, b) =1)  \tag{R0+} 
\end{align*}
\end{theorem}

\begin{proof}
\begin{itemize}
\item {Order-1: For any $a, b\in \mathbb{S}$, and $f, h, f', h'$ if $f \preceq h\, \&\, f'\preceq h'$, then $f(a, b) \leq h(a, b)$ and $f'(a, b) \leq h'(a, b)$ follow. This yields $f(a, b)\cdot f'(a, b) \leq h(a, b)\cdot h'(a, b)$ and therefore $ f\otimes f' \preceq h \otimes h'$. }
\item {Order-2: For any $a, b\in \mathbb{S}$, and $f, h, f', h'$ if $f \preceq h\, \&\, f'\preceq h'$, then $f(a, b) \leq h(a, b)$ and $f'(a, b) \leq h'(a, b)$ follow. Again for any $\alpha\in [0, 1]$, it follows that $\alpha f(a, b) \leq \alpha h(a, b)$ and $(1-\alpha)f'(a, b) \leq (1-\alpha)h'(a, b)$. This yields\\ $\alpha f(a, b) + (1-\alpha)f'(a, b) \leq \alpha h(a, b) +  (1-\alpha)h'(a, b)$. This proves $ f\oplus_{\alpha} f' \preceq h \oplus_{\alpha} h')$.}
\item {For any $a, b\in \mathbb{S}$ and $f\in W$, $f(a, b)\in [0, 1]$. So $f\preceq \top$ follows.}
\end{itemize}
The other parts follow directly.
 
\end{proof}

\begin{remark}
Clearly, wqRIF algebras are  ordered hemirings with additional operators as $\otimes$ distributes over $\oplus_{\alpha}$. Because of the inconsistent terminology used in the literature, some readers may want to replace \emph{hemiring} with \emph{semiring} in the last sentence.
\end{remark}

\subsection{Algebraic Systems of RIFs}

The first thing to be noted is that if a general RIF $\kap$ is defined on a GGS $\mathbb{S}$, then the parthood $\pc$ would be definable by it. The natural question is then: Would all operations definable on the set of all wqRIF on a GGS be definable on the set of all RIFs? 

Let $RIF(\mathbb{S})$ be the set of all RIFs on $\mathbb{S}$.

\begin{theorem}
\begin{itemize}
\item {The operations $\otimes$, and predicates in Definition \ref{operwqr} are well- defined in $RIF(\mathbb{S})$ (over a GGS $\mathbb{S}$) when $wqRIF(\mathbb{S})$ is uniformly replaced by $RIF(\mathbb{S})$.}
\item {The operations $\otimes$, and predicates in Definition \ref{operwqr} are well-defined in $RIF(\mathbb{S})$ over an abstract approximation system $\mathbb{S}$ when $wqRIF(\mathbb{S})$ is uniformly replaced by $RIF(\mathbb{S})$.}
\end{itemize}
\end{theorem}

\begin{proof}
The proof consists in checking that the operations are closed in $RIF(\mathbb{S})$. The failure of the other operations in a wqRIF algebra can also be traced. Counterexamples are easy.

\begin{itemize}
\item {It is necessary to prove that $f\otimes h$ is in $RIF(\mathbb{S})$ for any $f, h\in RIF(\mathbb{S})$. If for any $a, b\in \mathbb{S}$ $\pc ab$ holds, then $(f\otimes h)(a, b) = f(a, b) \cdot h(a, b) = 1$. Because of the possible values of $f$ and $h$, the converse holds as well This verifies R1. Again for the same $a, b$ and a $c\in \mathbb{S}$, $f(a, c)\leq f(b,c)$ and $h(a,c)\leq h(b,c)$ hold. So $(f(a,c) \cdot h(a,c))\leq (f(b, c)\cdot h(b, c))$ follows. This means R2 and R3 is satisfied by $f\otimes h$.  }
\item {Let $\alpha \in [0, 1]$, $f, h\in RIF(\mathbb{S})$, and $a, b\in \mathbb{S}$ be such that $\pc ab$ holds. Then $(f\oplus_{\alpha} h)(a, b) = \alpha f(a, b) + (1-\alpha) h(a, b) = \alpha +(1 - \alpha) = 1 $. This verifies R0. \emph{The converse is not provable in general. So R1 fails}. Again for the same $a, b$ and a $c\in \mathbb{S}$, $f(a, c)\leq f(b,c)$ and $h(a,c)\leq h(b,c)$ hold. From this $\alpha f(a,c)\leq \alpha f(b, c)$ and $(1 - \alpha) h(a, c)\leq (1 - \alpha) h(b, c)$ follows. Adding the two inequalities yields  $\alpha f(a, c) + (1 - \alpha) h(a, c)\leq \alpha f(b, c) + \leq (1 - \alpha) h(b. c)$. That is $ f\oplus_{\alpha} h (a, c) \leq f\oplus_{\alpha} h (b, c) $. R3 follows from this.}
\item {For any $a, b\in \mathbb{S}$ $\pc ab$ implies $\pc a^l b^l$. Therefore, $\pc a b $ implies $\sharp f (a, b) = f(a^l, b^l) = 1$. This verifies R0. Conversely, if $\sharp f (a, b) =1$, then $f(a^l, b^l) =1$. \emph{This means $\pc a^l b^l$ and it does not follow that $\pc ab$, and R1 can fail}. Again for the same $a$ and $b$, and a $c\in \mathbb{S}$, $f(a, c)\leq f(b, c)$. $\sharp f (a, c) = f(a^l, c^l) \leq f(b^l, c^l)$ holds because $\pc a^l b^l$ follows from $\pc ab$. This completes the verification of R3.}
\item {The well-definedness of $\flat$ fails in the same way as $\sharp$. }
\item {The relation $\preceq$ is obviously a subset of $(RIF(\mathbb{S}))^2$.}
\end{itemize}
\end{proof}

From the above, it can be shown that 

\begin{theorem}\label{rifalgm}
If $f\in RIF(\mathbb{S})$ and $h\in wqRIF(\mathbb{S})$, then $f\otimes h\in RIF(\mathbb{S})$.

$RIF(\mathbb{S} =\left\langle \underline{RIF(\mathbb{S}}, \otimes, \preceq, \right\rangle$ is a commutative partially ordered semigroup.
\end{theorem}

The algebraic system $RIF(\mathbb{S} =\left\langle \underline{RIF(\mathbb{S}}, \otimes, \preceq, \top \right\rangle$ will be referred to as a RIF algebra.

\subsection{Problems}

\begin{flushleft}
\textbf{Problem-1} 
\end{flushleft}

Given a pair of wqRIFs in a problem context, it maybe of interest to find one or a finite number of $\alpha$(s) that provides an optimal way of obtaining a combined wqRIF. This will be useful in the following situation:
\begin{itemize}
\item {Algorithm $X$ recommends wqRIFs $f$ and $h$.}
\item {A finite number of inclusion values $v_1, \ldots v_n$are recommended on the basis of expert information.  }
\item {This yields the subproblem of finding a single $\alpha \in[0, 1]$ possibly for which 
the \emph{total} error in the values of $\alpha f \oplus (1-\alpha )h$ from $v_i$ at the associated points is a minimum. The idea of total error being sum of squared errors or something else.}
\end{itemize}

\begin{flushleft}
\textbf{Problem-2} 
\end{flushleft}

Find minimal generating sets of a $\mathbb{A}$-wqRIF algebra $W$ for different $\mathbb{A}$. This problem is relevant for model selection. 

The following theorems that follow by inductive arguments from the properties of a wqRIF algebra suggest some strategies for solving the problem(s)

\begin{theorem}
Let the continued $k$-times product of $f$ be $f^k = f\otimes f\otimes \ldots f$.
If $\alpha_i \in Q\cap [0, 1]$, and $f_i\in wqRIF(S)$ for $i\in\{1, 2, \ldots, n\leq \infty\}$, then $\sum_{i=1}^n \alpha_i f^{n_i}_i \in wqRIF(S)$ whenever $\sum_{i=1}^n \alpha_i = 1$.
\end{theorem}

In other words, convex polynomials in one or more variables (wqRIFs) can represent a large number of wqRIFs. For every $f\in wqRIF(\mathbb{S})$, let $wqRIF_{f}(\mathbb{S})$ be the set of convex polynomials generated by $f$.

\begin{proposition}
For each $f\in wqRIF(\mathbb{S})$, and $\alpha\in[0, 1]$, $wqRIF_{f}(\mathbb{S})$ is closed under $\oplus_{\alpha}$ and $\otimes$. 
\end{proposition}

\begin{proposition}
If $f, h\in wqRIF(\mathbb{S})$ and $f\otimes h \in wqRIF_{f}(\mathbb{S})$, then $ wqRIF_{f\otimes h}(\mathbb{S}) \subseteq wqRIF_{f}(\mathbb{S})$.
\end{proposition}

\begin{proof}
The set of convex polynomials formed by $f$ is closed under the operations $\oplus_{\alpha}$ and $\otimes$. So the result follows.
\end{proof}

\begin{flushleft}
\textbf{Remarks} 
\end{flushleft}

In this research, concepts of high granular operator spaces and variants, contamination and data intrusion are explained. General rough inclusion functions are generalized to these frameworks and characterized. Of these generalized weak quasi rough inclusion functions are explored in depth and it is shown that they form hemirings with additional operations, while the algebras of generalized rough inclusion functions merely form ordered semigroups. These have potential applications to model selection, reducing contamination, and cluster validation from a rough perspective \cite{am2021c} among many other problems. In the last mentioned research, a rough framework for cluster validation is developed. The extent to which decontaminated measures can help in the methodology will be explored in a forthcoming paper. Interconnections between the two algebraic systems constructed in the present research lead to specific partial algebraic systems, and the extent to which these can be used to classify the generalized RIFs will also be part of future work.

\begin{flushleft}
\begin{small}
\textbf{Acknowledgement}: 
\end{small}
\end{flushleft}
This research is supported by a woman scientist grant of the department of science and technology.

\bibliographystyle{splncs.bst}
\bibliography{algroughf69flz}

\begin{thebibliography}{10}

\bibitem{sdrskt06}
{\'S}l{\c e}zak, D.:
\newblock {Association Reducts: Boolean Representation}.
\newblock In et. al, G.W., ed.: {RSKT 2006}. {LNAI 4062} (2006)  305--312

\bibitem{amst96}
Agrawal, R., Mannila, H., R., S., Toivonen, H., Verkamo, A.:
\newblock {Fast Discovery of Association Rules}.
\newblock In others, ed.: {Advances in Knowledge Discovery and Data Mining},
  MIT Press (1996)  307--328

\bibitem{am501}
Mani, A.:
\newblock {Algebraic Methods for Granular Rough Sets}.
\newblock In Mani, A., D{\"u}ntsch, I., Cattaneo, G., eds.: {Algebraic Methods
  in General Rough Sets}. {Trends in Mathematics}.
\newblock Birkhauser Basel (2018)  157--336

\bibitem{am240}
Mani, A.:
\newblock {Dialectics of Counting and The Mathematics of Vagueness}.
\newblock Transactions on Rough Sets \textbf{XV}(LNCS 7255) (2012)  122--180

\bibitem{gdu}
D{\"u}ntsch, I., Gediga, G.:
\newblock {Rough set data analysis: A road to non-invasive knowledge
  discovery}.
\newblock Methodos Publishers (2000)

\bibitem{am5559}
Mani, A.:
\newblock {Towards Student Centric Rough Concept Inventories}.
\newblock In Bello, R.,  et~al., eds.: {IJCRS' 2020}. Volume 12179 of {LNAI}.
\newblock Springer International (2020)  251--266

\bibitem{am9114}
Mani, A.:
\newblock {Knowledge and Consequence in AC Semantics for General Rough Sets}.
\newblock In Wang, G.,  et~al., eds.: {Thriving Rough Sets}. Volume 708 of
  {Studies in Computational Intelligence Series}.
\newblock Springer International (2017)  237--268

\bibitem{am2021c}
Mani, A.:
\newblock {General Rough Modeling of Cluster Analysis}.
\newblock In Ramanna, S.,  et~al., eds.: {Rough Sets: IJCRS-EUSFLAT 2021}.
  {LNAI 12872}.
\newblock Springer Nature (2021)  1--8

\bibitem{ppm2}
Pagliani, P., Chakraborty, M.:
\newblock {A Geometry of Approximation: Rough Set Theory: Logic, Algebra and
  Topology of Conceptual Patterns}.
\newblock Springer, Berlin (2008)

\bibitem{lp2011}
Polkowski, L.:
\newblock {Approximate Reasoning by Parts}.
\newblock Springer Verlag (2011)

\bibitem{tyl}
Lin, T.Y.:
\newblock {Granular Computing-1: The Concept of Granulation and Its Formal
  Model}.
\newblock Int. J. Granular Computing, Rough Sets and Int Systems \textbf{1}(1)
  (2009)  21--42

\bibitem{am99}
Mani, A.:
\newblock {Choice Inclusive General Rough Semantics}.
\newblock Information Sciences \textbf{181}(6) (2011)  1097--1115

\bibitem{am9411}
Mani, A.:
\newblock {Probabilities, Dependence and Rough Membership Functions}.
\newblock International Journal of Computers and Applications \textbf{39}(1)
  (2016)  17--35

\bibitem{am9969}
Mani, A.:
\newblock {Dialectical Rough Sets, Parthood and Figures of Opposition-I}.
\newblock Transactions on Rough Sets \textbf{XXI}(LNCS 10810) (2018)  96--141

\bibitem{yzm2012}
Yao, Y.Y., Zhang, N., Miao, D.:
\newblock {Set-Theoretic Approaches To Granular Computing}.
\newblock Fundamenta Informaticae \textbf{115} (2012)  247--264

\bibitem{hmy2019}
Mao, H., Hu, M., Yao, Y.Y.:
\newblock {Algebraic Approaches To Granular Computing}.
\newblock Granular Computing (2019)  1--13

\bibitem{am3930}
Mani, A.:
\newblock {Ontology, Rough Y-Systems and Dependence}.
\newblock Internat. J of Comp. Sci. and Appl. \textbf{11}(2) (2014)  114--136
  Special Issue of IJCSA on Computational Intelligence.

\bibitem{am3600}
Mani, A.:
\newblock {Contamination-Free Measures and Algebraic Operations}.
\newblock In: {Fuzzy Systems (FUZZ), 2013 IEEE International Conference on},
  IEEE (2013)  1--8

\bibitem{bu}
Burmeister, P.:
\newblock {A Model-Theoretic Oriented Approach to Partial Algebras}.
\newblock Akademie-Verlag (1986, 2002)

\bibitem{lj}
Ljapin, E.S.:
\newblock {Partial Algebras and Their Applications}.
\newblock Academic, Kluwer (1996)

\bibitem{cptrs19}
Samanta, P., Chakraborty, M.K.:
\newblock {Interface of Rough Set Systems and ModaLogics: A Survey}.
\newblock Transactions on Rough Sets \textbf{XIX, LNCS 8988} (2015)  114--137

\bibitem{pp2018}
Pagliani, P.:
\newblock {Three Lessons on the Topological and Algebraic Hidden Core of Rough
  Set Theory}.
\newblock In Mani, A., D{\"u}ntsch, I., Cattaneo, G., eds.: {Algebraic Methods
  in General Rough Sets}. {Trends in Mathematics}.
\newblock Springer International (2018)  337--415

\bibitem{cc5}
Cattaneo, G., Ciucci, D.:
\newblock {Lattices With Interior and Closure Operators and Abstract
  Approximation Spaces}.
\newblock In Peters, J.F.,  et~al., eds.: {Transactions on Rough Sets X, LNCS
  5656}.
\newblock Springer (2009)  67--116

\bibitem{gcd2018}
Cattaneo, G., Ciucci, D.:
\newblock {Algebraic Methods for Orthopairs and induced Rough Approximation
  Spaces}.
\newblock In Mani, A., D{\"u}ntsch, I., Cattaneo, G., eds.: {Algebraic Methods
  in General Rough Sets}.
\newblock Birkhauser Basel (2018)  553--640

\bibitem{cd3}
Ciucci, D.:
\newblock {Approximation Algebra and Framework}.
\newblock Fundamenta Informaticae \textbf{94} (2009)  147--161

\bibitem{gc2018}
Cattaneo, G.:
\newblock {Algebraic Methods for Rough Approximation Spaces by Lattice
  Interior--closure Operations}.
\newblock In Mani, A., D{\"u}ntsch, I., Cattaneo, G., eds.: {Algebraic Methods
  in General Rough Sets}. {Trends in Mathematics}.
\newblock Springer International (2018)  13--156

\bibitem{am6900}
Mani, A.:
\newblock {Pure Rough Mereology and Counting}.
\newblock In: {WIECON,2016}, IEEXPlore (2016)  1--8

\bibitem{am9204}
Mani, A.:
\newblock {Generalized Ideals and Co-Granular Rough Sets}.
\newblock In Polkowski, L.,  et~al., eds.: {Rough Sets, Part 2, IJCRS,2017 }.
  {LNAI 10314}.
\newblock Springer International (2017)  23--42

\bibitem{am1800}
Mani, A.:
\newblock {Axiomatic Approach to Granular Correspondences}.
\newblock In Li, T.,  et~al., eds.: {Proceedings of RSKT'2012}. Volume LNAI
  7414., Springer-Verlag (2012)  482--487

\bibitem{am6999}
Mani, A.:
\newblock {Antichain Based Semantics for Rough Sets}.
\newblock In Ciucci, D., Wang, G., Mitra, S., Wu, W., eds.: {RSKT 2015},
  Springer-Verlag (2015)  319--330

\bibitem{am5550}
Mani, A.:
\newblock {Functional Extensions of Knowledge Representation in General Rough
  Sets}.
\newblock In Bello, R.,  et~al., eds.: {IJCRS' 2020}. Volume 12179 of {LNAI}.
\newblock Springer International (2020)  19--34

\bibitem{it2}
Iwinski, T.B.:
\newblock {Rough Orders and Rough Concepts}.
\newblock Bull. Pol. Acad. Sci (Math) \textbf{(3--4)} (1988)  187--192

\bibitem{av}
Varzi, A.:
\newblock {Parts, Wholes and Part-Whole Relations: The Prospects of
  Mereotopology}.
\newblock Data and Knowledge Engineering \textbf{20} (1996)  259--286

\bibitem{vie}
Vieu, L.:
\newblock {On The Transitivity of Functional Parthood}.
\newblock Applied Ontology \textbf{1}(2) (2007)  147--155

\bibitem{ur}
Urbaniak, R.:
\newblock {Lesniewski's Systems of Logic and Mereology; History and
  Re-Evaluation}.
\newblock PhD thesis, Department of Philosophy, Univ of Calgary (2008)

\bibitem{rgac15}
Gruszczy{\'n}ski, R., Varzi, A.:
\newblock {Mereology Then and Now}.
\newblock Logic and Logical Philosophy \textbf{24} (2015)  409--427

\bibitem{seibtj2015}
Seibt, J.:
\newblock {Transitivity}.
\newblock In Burkhardt, H., Seibt, J., Imaguire, G., Gerogiorgakis, S., eds.:
  {Handbook of Mereology}.
\newblock Philosophia Verlag, Germany (2017)  570--579

\bibitem{skaj2016}
Skowron, A., Jankowski, A.:
\newblock {Rough Sets and Interactive Granular Computing}.
\newblock Fundamenta Informaticae \textbf{147} (2016)  371--385

\bibitem{skajsd2016}
Skowron, A., Jankowski, A., Dutta, S.:
\newblock {Interactive granular computing}.
\newblock Granular Computing \textbf{1}(2) (2016)  95--113

\bibitem{am5586}
Mani, A.:
\newblock {Comparative Approaches to Granularity in General Rough Sets}.
\newblock In Bello, R.,  et~al., eds.: {IJCRS 2020}. Volume 12179 of {LNAI}.
\newblock Springer (2020)  500--518

\bibitem{ag2009}
Gomolinska, A.:
\newblock {Rough Approximation Based on Weak q-RIFs}.
\newblock Transactions on Rough Sets \textbf{X} (2009)  117--135

\bibitem{ag3}
Gomolinska, A.:
\newblock {On Certain Rough Inclusion Functions}.
\newblock In Peters, J.F.,  et~al., eds.: {Transactions on Rough Sets IX, LNCS
  5390}.
\newblock Springer Verlag (2008)  35--55

\bibitem{ss2010}
Skowron, A., Stepaniuk, J.:
\newblock {Approximation Spaces in Rough-Granular Computing}.
\newblock Fundamenta Informaticae \textbf{100} (2010)  141--157

\bibitem{zpb}
Pawlak, Z.:
\newblock {Rough Sets: Theoretical Aspects of Reasoning About Data}.
\newblock Kluwer Academic Publishers, Dodrecht (1991)

\bibitem{zw}
Ziarko, W.:
\newblock {Variable Precision Rough Set Model}.
\newblock J. of Computer and System Sciences \textbf{46} (1993)  39--59

\bibitem{ss1}
Skowron, A., Stepaniuk, O.:
\newblock {Tolerance Approximation Spaces}.
\newblock Fundamenta Informaticae \textbf{27} (1996)  245--253

\bibitem{yec2017}
Syau, Y.R., Lin, E.B., Liau, C.j.:
\newblock {Neighborhood Systems and Variable Precision Generalized Rough Sets}.
\newblock Fundamenta Informaticae \textbf{153} (2017)  271--290

\bibitem{jzd2003}
Liang, J., Shi, Z., Li, D.:
\newblock {Applications of Inclusion Degree in Rough Set Theory}.
\newblock International Journal of Computational Cognition: YangSky
  \textbf{1}(2) (June 2003)  67--78

\bibitem{chen2014}
Chen, D., Li, W., Zhang, X., Kwong, S.:
\newblock {Evidence Theory Based Numerical Algorithms Of Attribute Reduction
  With Neighborhood Covering Rough Sets}.
\newblock Int. J. Approx. Reasoning \textbf{55} (2014)  908--923

\bibitem{wu2002}
Wu, W.Z., Leung, Y., Zhang, W.:
\newblock {Connections Between Rough Set Theory and Dempster--shafer Theory of
  Evidence}.
\newblock Int. J. General Systems \textbf{31} (2002)  405--430

\bibitem{skg1994}
Skowron, A., Grzymala-Busse, J.:
\newblock {From Rough Set Theory to Evidence Theory}.
\newblock In Yager, R.,  et~al., eds.: {Advances in the Dempster--Shafer Theory
  of Evidence}.
\newblock Wiley (1994)  193--236

\end{thebibliography}
\end{document}